\documentclass[10pt,a4paper,oneside]{article}

\usepackage{graphicx}
\usepackage{epstopdf}
\usepackage[noblocks]{authblk}
\usepackage{amssymb}
\usepackage{amsmath}
\usepackage{amsfonts}
\usepackage{mathtools}
\usepackage{amsfonts}
\usepackage{color}
\usepackage{listings}
\usepackage{fancyvrb}
\usepackage{float}
\usepackage{framed}
\usepackage{multirow}
\usepackage{url}  
\usepackage{multicol}
\usepackage{subcaption}
\usepackage{algorithm}
\usepackage{algpseudocode}
\usepackage{pifont}
\usepackage{algorithm}
\usepackage{colortbl}
\usepackage{framed}
\usepackage{listings}
\usepackage{bm}
\usepackage{eurosym}
\usepackage{amssymb}
\usepackage{suffix}
\usepackage{tikz}
\usetikzlibrary{matrix}

\usepackage[a4paper,left=3.5cm,right=3.5cm,top=3.5cm,bottom=3.5cm]{geometry}

\newtheorem{theorem}{Theorem}[section]
\newtheorem{proposition}[theorem]{Proposition}

\newenvironment{proof}{\noindent{\bf Proof:}}{\hfill$\square$}

\DeclarePairedDelimiterX\MeijerM[3]{\lparen}{\rparen}%
{\begin{smallmatrix}#1 \\ #2\end{smallmatrix}\delimsize\vert\,#3}

\newcommand\MeijerG[8][]{%
  G^{\,#2,#3}_{#4,#5}\MeijerM[#1]{#6}{#7}{#8}}

\WithSuffix\newcommand\MeijerG*[7]{%
  G^{\,#1,#2}_{#3,#4}\MeijerM*{#5}{#6}{#7}}

\title{Optimal binning: mathematical programming formulation}
\author{Guillermo Navas-Palencia\\ \texttt{g.navas.palencia@gmail.com}}
\date{\today\footnote{This version contains a new objective function and the maximum p-value constraint for the continuous target. The initial version date is January 22, 2020.}}

\begin{document}
\maketitle

\begin{abstract}
The optimal binning is the optimal discretization of a variable into bins given a discrete or continuous numeric target. We present a rigorous and extensible mathematical programming formulation to solve the optimal binning problem for a binary, continuous and multi-class target type, incorporating constraints not previously addressed. For all three target types, we introduce a convex mixed-integer programming formulation. Several algorithmic enhancements, such as automatic determination of the most suitable monotonic trend via a Machine-Learning-based classifier, and implementation aspects are thoughtfully discussed. The new mathematical programming formulations are carefully implemented in the open-source python library OptBinning.
\end{abstract}

\section{Introduction}
Binning (grouping or bucketing) is a technique to discretize the values of a continuous variable into bins (groups or buckets). From a modeling perspective, the binning technique may address prevalent data issues such as the handling of missing values, the presence of outliers and statistical noise, and data scaling. Furthermore, the binning process is a valuable interpretable tool to enhance the understanding of the nonlinear dependence between a variable and a given target while reducing the model complexity. Ultimately, resulting bins can be used to perform data transformations.

Binning techniques are extensively used in machine learning applications, exploratory data analysis and as an algorithm to speed up learning tasks; recently, binning has been applied to accelerate learning in gradient boosting decision tree \cite{Ke2017}. In particular, binning is widely used in credit risk modeling, being an essential tool for credit scorecard modeling to maximize differentiation between high-risk and low-risk observations.

There are several unsupervised and supervised binning techniques. Common unsupervised techniques are equal-width and equal-size or equal-frequency interval binning. On the other hand, well-known supervised techniques based on merging are Monotone Adjacent Pooling Algorithm (MAPA), also known as Maximum Likelihood Monotone Coarse Classifier (MLMCC) \cite{Thomas2002} and ChiMerge \cite{Kerber1992}, whereas other techniques based on decision trees are CART \cite{Breiman1984}, Minimum Description Length Principle (MDLP) \cite{Fayyad1993} and, more recently, condition inference trees (CTREE) \cite{Hothorn2006}.

The binning process might require to satisfy certain constraints. These constraints might range from requiring a minimum number of records per bin to monotonicity constraints. This variant of the binning process is known as the optimal binning process. The optimal binning is generally solved by iteratively merging an initial granular discretization until imposed constraints are satisfied. Performing this fine-tuning manually is likely to be unsatisfactory as the number of constraints increases, leading to suboptimal or even infeasible solutions. However, we note that this manual adjustment has been encouraged by some authors \cite{Siddiqi2005}, legitimating the existing interplay of ``art and science'' in the binning process.

There are various commercial software tools for solving the optimal binning problem\footnote{To the author's knowledge, at the time of writing, these tools are restricted to the problem of discretizing a variable with respect to a binary target.}. Software IBM SPSS and the MATLAB Financial Toolbox, use MDLP and MAPA as default algorithm, respectively. The most advanced tool to solve the optimal binning problem is available in the SAS Enterprise Miner software. A limited description of the proprietary algorithm can be found in \cite{Oliveira2008}, where two mixed-integer programming (MIP) formulations are sketched: a mixed-integer linear programming (MILP) formulation to obtain a fast probably suboptimal solution, and a mixed-integer nonlinear programming (MINLP) formulation to obtain an optimal solution. The suboptimal formulation is the default method due to computational time limitations (MILP techniques are considerably more mature). We note that the SAS implementation allows most of the constraints required in credit risk modeling, becoming an industry standard. Besides, there exist a few open-source solutions, but the existing gap comparing to the commercial options in terms of capabilities is still significant. Among the available alternatives, we mention the MATLAB implementation of the monotone optimal binning in \cite{Mironchyk2017}, and the R specialized packages \texttt{smbinning} \cite{smbinning}, relying on CTREE, and \texttt{MOB} \cite{mob}, which merely include basic functionalities.

In this paper, we develop a rigorous and extensible mathematical programming formulation for solving the optimal binning problem. This general formulation can efficiently handle binary, continuous, and multi-class target type. The presented formulations incorporate the constraints generally required to produce a good binning \cite{Siddiqi2005}, and new constraints not previously addressed. For all three target types, we introduce a convex mixed-integer programming formulation, ranging from a integer linear programming (ILP) formulation for the simplest cases to a mixed-integer quadratic programming (MIQP) formulation for those cases adding more involved constraints.

The remainder of the paper is organized as follows. Section \ref{section_mip_formulation} introduces our general problem formulation and the corresponding mixed-integer programming formulation for each supported target. We focus on the formulation for binary target, investigating various formulation variants.  Then, in Section \ref{section_algorithms} we discuss in detail several algorithmic aspects such as the automatic determination of the optimal monotonic trend and the development of presolving algorithms to efficiently solve large size instances. Section \ref{section_experiments} includes experiments with real-world datasets and compares the performance of supported solvers for large size instances. Finally, in Section \ref{section_conclusions}, we present our conclusions and discuss possible research directions.

\section{Mathematical programming formulation}\label{section_mip_formulation}

The optimal binning process comprises two steps: A pre-binning process that generates an initial granular discretization, and a subsequent refinement or optimization to satisfy imposed constraints. The pre-binning process uses, for example, a decision tree algorithm to calculate the initial split points. The resulting $m$ split points are sorted in strictly ascending order, $s_1 < s_2 < \ldots < s_m$ to create $n = m + 1$ pre-bins. These pre-bins are defined by the intervals $(-\infty, s_1), [s_1, s_2), \ldots, [s_m, \infty)$.

Given $n$ pre-bins, the decision variables consist of a binary lower triangular matrix (indicator variables) of size $n$, $X_{ij} \in \{0, 1\}, \forall (i, j) \in \{1, \ldots, n: i \ge j\}$. The starting point is a diagonal matrix, meaning that initially all pre-bins are selected. A basic feasible solution must satisfy the following constraints:
\begin{itemize}
\item All pre-bins are either isolated or merged to create a larger bin interval, but cannot be erased. Each column must contain exactly one $1$.
\begin{equation}\label{def_ct_sum_columns_one}
\sum_{i=1}^n X_{ij} = 1, \quad j = 1, \ldots, n.
\end{equation}
\item Only consecutive pre-bins can be merged. Continuity by rows, no $0-1$ gaps are allowed.
\begin{equation}\label{def_ct_flow_continuity}
X_{ij} \le X_{ij+1}, \quad i = 1, \ldots, n; \; j = 1, \ldots, i-1.
\end{equation}
\item A solution has a last bin interval of the form $[s_k, \infty)$, for $k \le n$. The binary decision variable $X_{nn} = 1$.
\end{itemize}

To clarify, Figure \ref{basic_diagram} shows an example of a feasible solution. In this example, pre-bins corresponding to split points $(s_2, s_3, s_4)$ and $(s_5, s_6)$ are merged, thus having an optimal binning with bin intervals $(-\infty, s_1), [s_1, s_4), [s_4, s_6), [s_6, \infty)$.

\begin{figure}[ht]
	\centering
	\includegraphics[scale=0.33]{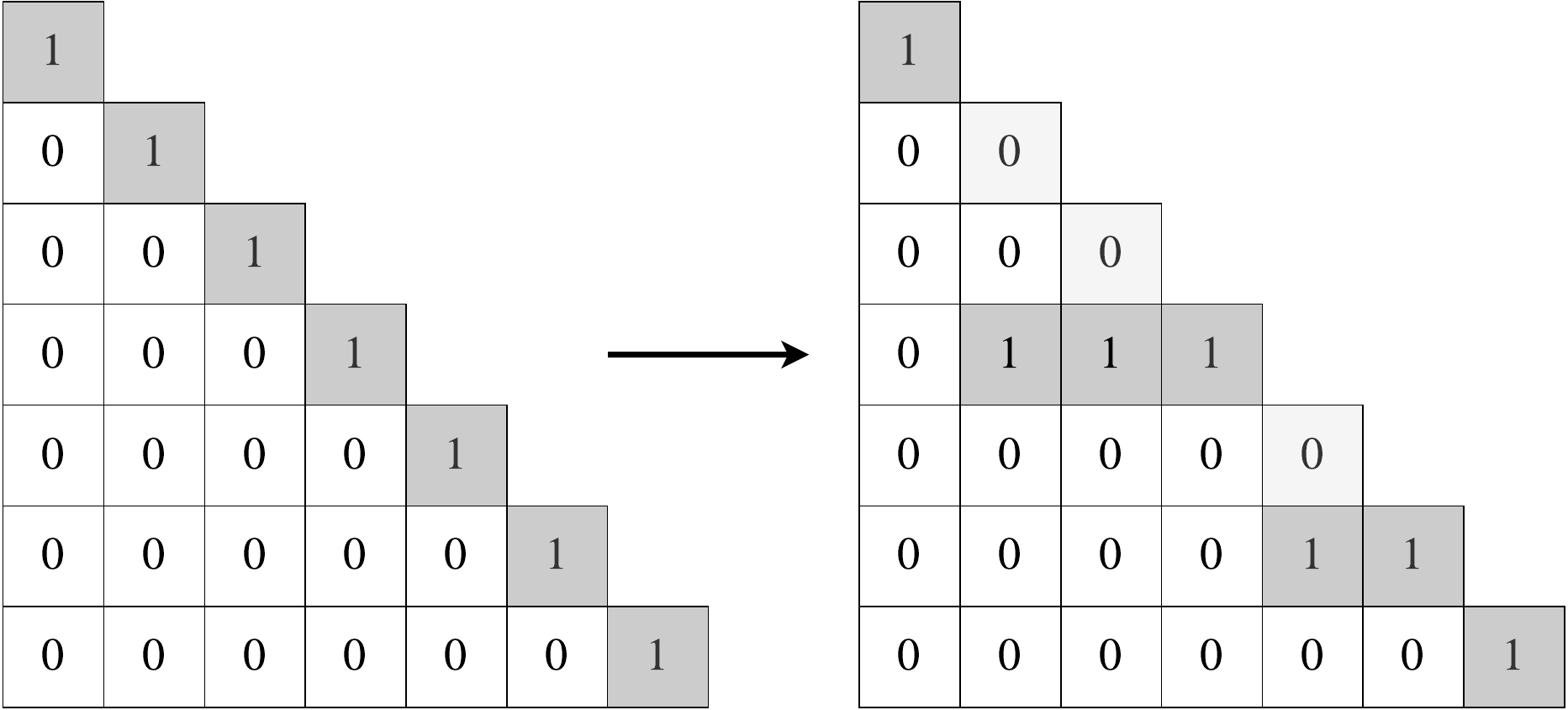}
	\caption{Lower triangular matrix $X$. Initial solution after pre-binning (left). Optimal solution with 4 bins after merging pre-bins (right).}
	\label{basic_diagram}
\end{figure}

The described problem can be seen as a generalized assignment problem. A direct formulation of metrics involving ratios such as the mean or most of the divergence measures on merged bins leads to a non-convex MINLP formulation, due to the ratio of sums of binary variables. Solving non-convex MINLP problems to optimality is a challenging task requiring the use of global MINLP solvers, especially for large size instances.

Investigating the binary lower triangular matrix in Figure \ref{basic_diagram}, it can be observed, by analyzing the constraints in Equations (\ref{def_ct_sum_columns_one}) and (\ref{def_ct_flow_continuity}) imposing continuity by rows, that a feasible solution is entirely characterized by the position of the first $1$ for each row. This observation permits the pre-computation of the set of possible solutions by rows, obtaining an aggregated matrix with the shape of $X$ for each involved metric. Consequently, the non-convex objective function and constraints are linearized, resulting in a convex formulation by exploiting problem information. Using this reformulation, we shall see that the definition of constraints for binary, continuous and multi-class target are almost analogous.

\subsection{Mixed-integer programming formulation for binary target}
Given a binary target $y$ used to discretize a variable $x$ into $n$ bins, we define the normalized count of non-events (NE) $p_i$, case $y=0$, and events (E) $q_i$, case $y=1$, for each bin $i$ as
\begin{equation*}
p_i = \frac{r_i^{NE}}{r_T^{NE}}, \quad q_i = \frac{r_i^E}{r_T^E},
\end{equation*}
where $r_i^{NE}$, $r_i^{E}$, $r_T^{NE}$ and $r_T^E$ are the number of non-event and event records per bin, and the total number of non-event records and event records, respectively. Next, we define the Weight of Evidence (WoE) and event rate ($D$) for each bin,
\begin{equation*}
\text{WoE}_i = \log\left(\frac{r_i^{NE}/ r_T^{NE}}{r_i^E / r_T^E}\right), \quad D_i = \frac{r_i^E}{r_i^E + r_i^{NE}}.
\end{equation*}
The Weight of Evidence $\text{WoE}_i$ and event rate $D_i$ for each bin are related by means of the functional equations
\begin{align*}
\text{WoE}_i &= \log\left(\frac{1 - D_i}{D_i}\right) + \log\left(\frac{r_T^E}{ r_T^{NE}}\right) = \log\left(\frac{r_T^E}{ r_T^{NE}}\right) - \text{logit}(D_i)\\
D_i &= \left(1 + \frac{r_T^{NE}}{r_T^{E}} e^{\text{WoE}_i}\right)^{-1} = \left(1 + e^{\text{WoE}_i - \log\left(\frac{r_T^{E}}{r_T^{NE}}\right)}\right)^{-1},
\end{align*}
where $D_i$ can be characterized as a logistic function of $\text{WoE}_i$, and  $\text{WoE}_i$ can be expressed in terms of the logit function of $D_i$. This shows that WoE is inversely related to the event rate. The constant term $\log(r_T^{E} / r_T^{NE})$ is the log ratio of the total number of event and the total number of non-events. 

Divergence measures serve to assess the discriminant power of a binning solution. The Jeffreys' divergence \cite{Jeffreys1946}, also known as Information Value (IV) within the credit risk industry, is a symmetric measure expressible in terms of the Kullback-Leibler divergence $D_{KL}(P || Q)$ \cite{kullback1951} defined by
\begin{equation*}
J(P|| Q) = IV = D_{KL}(P || Q) + D_{KL}(Q || P) = \sum_{i=1}^n (p_i - q_i) \log \left(\frac{p_i}{q_i}\right).
\end{equation*}
The IV statistic is unbounded, but some authors have proposed rules of thumb to settings quality thresholds \cite{Siddiqi2005}. Alternatively, the Jensen-Shannon divergence is a bounded symmetric measure also expressible in terms of the Kullback-Leibler divergence
\begin{equation*}
JSD(P || Q) = \frac{1}{2}\left(D(P || M) + D(Q || M)\right), \quad M = \frac{1}{2}(P + Q),
\end{equation*}
and bounded by $JSD(P||Q) \in [0, \log(2)]$. Note that these divergence measures cannot be computed when $r_i^{NE} = 0$ and/or $r_i^E = 0$. Other divergences measures without this limitation are described in \cite{Zeng2013}.

A good binning algorithm for binary target should be characterized by the following properties \cite{Siddiqi2005}:
\begin{enumerate}
\item Missing values are binned separately.
\item Each bin should contain at least $5\%$ observations.
\item No bins should have $0$ non-events or events records.
\end{enumerate}

Property 1 is adequately addressed in many implementations, where missing and special values are incorporated as additional bins after the optimal binning terminates. Property 2 is a usual constraint to enforce representativeness. Property 3 is required to compute the above divergence measures.

Let us define the parameters of the mathematical programming formulation:
\begin{align*}
& n \in \mathbb{N} && \textrm{number of pre-bins.}\\ 
& r^{NE}_T \in \mathbb{N} && \textrm{total number of non-event records.}\\
& r^E_T \in \mathbb{N} && \textrm{total number of event records.}\\
& r^{NE}_i \in \mathbb{N}      && \textrm{number of non-event records per pre-bin.}\\
& r^E_i \in \mathbb{N}      && \textrm{number of event records per pre-bin.}\\
& r_i = r^{NE}_i + r^E_i        && \textrm{number of records per pre-bin.}\\
& r^{NE}_{\min} \in \mathbb{N}  && \textrm{minimum number of non-event records per bins.}\\
& r^{NE}_{\max} \in \mathbb{N}  && \textrm{maximum number of non-event records per bins.}\\
& r^E_{\min} \in \mathbb{N}  && \textrm{minimum number of event records per bins.}\\
& r^E_{\max} \in \mathbb{N}  && \textrm{maximum number of event records per bins.}\\
& b_{\min} \in \mathbb{N}  && \textrm{minimum number of bins.}\\
& b_{\max} \in \mathbb{N}  && \textrm{maximum number of bins.}
\end{align*}

The objective function is to maximize the discriminant power among bins, therefore, maximize a divergence measure. The IV can be computed using the described parameters and the decision variables $X_{ij}$, yielding
\begin{equation*}
IV = \sum_{i=1}^n \left(\sum_{j=1}^i \left(\frac{r_z^{NE}}{r^{NE}_T} - \frac{r_z^E}{r^E_T}\right)X_{ij}\right) \log\left(\frac{\sum_{j=1}^i r_j^{NE} / r_T^{NE} X_{ij}}{\sum_{j=1}^i r_j^E / r_T^E X_{ij}}\right),
\end{equation*}
The IV is the sum of the IV contributions per bin, i.e., the sum by rows. As previously stated, given the constraints in Equations (\ref{def_ct_sum_columns_one}) and (\ref{def_ct_flow_continuity}), an aggregated low triangular matrix $V_{ij} \in \mathbb{R}^+_0, \forall (i, j) \in \{1, \ldots, n: i \ge j\}$ with all possible IV values from bin merges can be pre-computed as follows
\begin{equation}\label{iv_obj_params}
V_{ij} = \left(\sum_{z=j}^i \frac{r_z^{NE}}{r^{NE}_T} - \frac{r_z^E}{r^E_T}\right) \log \left(\frac{\sum_{z=j}^i r_z^{NE} / r_T^{NE}}{\sum_{z=j}^i r_z^E / r_T^E}\right), \quad i=1, \ldots, n;\; j=1, \ldots, i.
\end{equation}
The optimal IV for each bin is determined by using the remarked observation that a solution is characterized by the position of the first $1$ for each row, thus, using the continuity constraint in (\ref{def_ct_flow_continuity}), we obtain
\begin{equation}
V_{i\cdot} = V_{i1} X_{i1} + \sum_{j=2}^i V_{ij} (X_{ij} - X_{ij-1}) \Longleftrightarrow V_{ii} X_{ii} + \sum_{j=1}^{i-1} (V_{ij} - V_{ij+1})X_{ij}.
\label{iv_obj_definition}
\end{equation}
for $i=1, \ldots, n$. The latter formulation is preferred to reduce the fill-in of the matrix of constraints. Similarly, a lower triangular matrix of event rates $D_{ij} \in [0, 1], \forall (i, j) \in \{1, \ldots, n: i \ge j\}$ can be pre-computed as follows
\begin{equation*}
D_{ij} = \frac{\sum_{z=j}^i r_z^E}{\sum_{z=j}^i r_z}, \quad i=1, \ldots, n;\; j=1, \ldots, i.
\end{equation*}

The ILP formulation, with no additional constraints such as monotonicity constraints, can be stated as follows
\begin{subequations}
\begin{align}
\underset{X}{\text{max}} \quad & \sum_{i=1}^n V_{ii} X_{ii} + \sum_{j=1}^{i-1} (V_{ij} - V_{ij+1}) X_{ij} \label{objective}\\
\text{s.t.} \quad &  \sum_{i=j}^{n} X_{ij} = 1, & j = 1,\ldots, n\label{sum_columns_one}\\
\quad & X_{ij} - X_{ij+1} \le 0 , & i = 1, \ldots, n; \; j = 1, \ldots, i-1\label{flow_continuity}\\
\quad & b_{\min} \le \sum_{i=1}^n X_{ii} \le b_{\max} \label{min_max_bins}\\
\quad & r_{\min} X_{ii} \le \sum_{j=1}^i r_j X_{ij} \le r_{\max} X_{ii}, & i = 1, \ldots, n \label{min_max_bin_size}\\
\quad & r^{NE}_{\min} X_{ii} \le \sum_{j=1}^i r^{NE}_j X_{ij} \le r^{NE}_{\max} X_{ii}, & i = 1, \ldots, n \label{min_max_bin_n_nonevent}\\
\quad & r^E_{\min} X_{ii} \le \sum_{j=1}^i r^E_j X_{ij} \le r^E_{\max} X_{ii}, & i = 1, \ldots, n \label{min_max_bin_n_event}\\
\quad & X_{ij} \in \{0, 1\}, & \forall (i, j) \in \{1, \ldots, n: i \ge j\}\label{binary_indicator}
\end{align}
\end{subequations}

Apart from the constraints (\ref{sum_columns_one}) and (\ref{flow_continuity}) already described, constraint (\ref{min_max_bins}) imposes a lower and upper bound on the number of bins. Other range constraints (\ref{min_max_bin_size}-\ref{min_max_bin_n_event}) limit the number of total, non-event and event records per bin. Note that to increase sparsity, the range constraints are not implemented following the standard formulation to avoid having the data twice in the model. For example, constraint (\ref{min_max_bins}) is replaced by
\begin{equation*}
d + \sum_{i=1}^n X_{ii} - b_{\max} =0, \quad 0 \le d \le b_{\max} - b_{\min}.
\end{equation*}

\subsubsection{Monotonicity constraints}\label{section_monotonicity}
Monotonicity constraints between the event rates of consecutive bins can be imposed to ensure legal compliance and business constraints. Three types of monotonic trends are considered: the usual ascending/descending, and two types of unimodal forms, concave/convex and peak/valley. This modeling flexibility can help to capture overlooked or unexpected patterns, providing new insights to enrich models. Note that work in \cite{Oliveira2008} uses the WoE approach instead.

Applying Equation (\ref{iv_obj_definition}), the optimal event rate for each bin is given by
\begin{equation}
D_{i\cdot} = D_{ii} X_{ii} + \sum_{j=1}^{i-1} (D_{ij} - D_{ij+1})X_{ij}, \quad i=1, \ldots, n.\label{event_rate_definition}
\end{equation}

\paragraph{Monotonic trend: ascending and descending.} The formulation for monotonic ascending trend can be stated as follows,
\begin{align*}
&D_{zz} X_{zz} + \sum_{j=1}^{z-1}(D_{zj} - D_{z j+1}) X_{zj} + \beta(X_{ii} + X_{zz} - 1)\nonumber\\ &\le 1 + (D_{ii} - 1) X_{ii} + \sum_{j=1}^{i - 1} (D_{ij} - D_{ij+1})X_{ij}, \quad i=2, \ldots, n; \; z=1,\ldots i-1.
\end{align*}
The term $1 + (D_{ii} - 1) X_{ii}$ or simply $1 - X_{ii}$, is used to ensure that event rates are in $[0, 1]$, and the ascending constraint is satisfied even if bin $i$ is not selected. Note that this is a big-$M$ formulation $M + (D_{ii} - M) X_{ii}$ using $M=1$, which suffices given $D \in [0, 1]$, however, a tighter (non integer) $M = \max(\{D_{ij}: i=1, \ldots, n;\; i \ge j\})$, can be used instead. The parameter $\beta$ is the minimum event rate difference between consecutive bins. The term $\beta(X_{ii} + X_{zz} - 1)$ is required to ensure that the difference between two selected bins $i$ and $z$ is greater or equal than $\beta$. Similarly, for the descending constraint,
\begin{align*}
&D_{ii} X_{ii} + \sum_{j=1}^{i-1}(D_{ij} - D_{i j+1})X_{ij} + \beta(X_{ii} + X_{zz} - 1)\nonumber\\ &\le 1 + (D_{zz} - 1) X_{zz} + \sum_{j=1}^{z - 1} (D_{zj} - D_{zj+1})X_{zj}, \quad i=2, \ldots, n; \; z=1,\ldots i-1.
\end{align*}

\paragraph{Monotonic trend: concave and convex.}
The concave and convex trend can be achieved by taking the definition of concavity/convexity on equally spaced points:
\begin{align*}
-x_{i+1} + 2 x_i - x_{i-1} &\ge 0 && \textrm{concave}\\
x_{i+1} - 2 x_i + x_{i-1} &\ge 0 && \textrm{convex}
\end{align*}
Thus, replacing Equation (\ref{event_rate_definition}) in the previous definition of concavity we obtain the concave trend constraints,
\begin{align*}
&-\left(D_{ii}X_{ii} + \sum_{z=1}^{i-1}(D_{iz} - D_{iz+1}) X_{iz}\right) + 2 \left(D_{jj}X_{jj} + \sum_{z=1}^{j-1}(D_{jz} - D_{jz+1}) X_{jz}\right)\nonumber\\
&-\left(D_{kk}X_{kk} + \sum_{z=1}^{k-1}(D_{kz} - D_{kz+1}) X_{kz}\right) \ge X_{ii} + X_{jj} + X_{kk} - 3,
\end{align*}
for $i=3, \ldots n$; $j=2, \ldots, i-1$ and $k=1, \ldots, j-1$. Similarly, for convex trend we get
\begin{align*}
&\left(D_{ii}X_{ii} + \sum_{z=1}^{i-1}(D_{iz} - D_{iz+1}) X_{iz}\right) - 2 \left(D_{jj}X_{jj} + \sum_{z=1}^{j-1}(D_{jz} - D_{jz+1}) X_{jz}\right)\nonumber\\
&\left(D_{kk}X_{kk} + \sum_{z=1}^{k-1}(D_{kz} - D_{kz+1}) X_{kz}\right) \ge X_{ii} + X_{jj} + X_{kk} - 3,
\end{align*}
for $i=3, \ldots n$; $j=2, \ldots, i-1$ and $k=1, \ldots, j-1$. Note that term $X_{ii} + X_{jj} + X_{kk} - 3$ is used the preserve redundancy of constraints when not all bins $i$, $j$ and $k$ are selected, given that $D \in [0, 1]$.

\paragraph{Monotonic trend: peak and valley.}
The peak and valley trend\footnote{In some commercial tools, peak and valley trend are called inverse U-shaped and U-shaped, respectively.} define an event rate function exhibiting a single trend change or reversal. The optimal trend change position is determined by using disjoint constraints, which can be linearized using auxiliary binary variables. The resulting additional constraints are as follows,
\begin{subequations}
\begin{align}
i - n (1 - y_i) &\le t \le i + n y_i, & i = 1, \ldots, n\label{disjoint_constraint}\\
t &\in [0, n]\\
y_i &\in \{0, 1\}, & i = 1, \ldots, n\label{disjoint_aux_variables}
\end{align}
\end{subequations}
where $t$ is the position of the optimal trend change bin, $y_i$ are auxiliary binary variables and $n$ in (\ref{disjoint_constraint}) is the smallest big-$M$ value for this formulation while preserving the redundancy of constraints. Furthermore, for the peak trend we incorporate the following constraints,
\begin{subequations}
\begin{align*}
&y_i + y_z + 1 + (D_{zz} - 1) X_{zz} + \sum_{j=1}^{z-1} (D_{zj} - D_{zj+1}) X_{zj}\nonumber\\
&\ge D_{ii} X_{ii} + \sum_{j=1}^{i-1}(D_{ij} - D_{ij+1}) X_{ij}, \quad i=2, \ldots, n; \; z=1, \ldots, i - 1,\\
& 2 - y_i - y_z + 1 + (D_{ii} - 1) X_{ii} + \sum_{j=1}^{i-1}(D_{ij} - D_{ij+1}) X_{ij}\nonumber\\
&\ge D_{zz} X_{zz} + \sum_{j=1}^{z-1} (D_{zj} - D_{zj+1}) X_{zj}, \quad i=2, \ldots, n; \; z=1, \ldots, i - 1.
\end{align*}
\end{subequations}
Similarly, for the valley trend we include,
\begin{subequations}
\begin{align*}
&y_i + y_z + 1 + (D_{ii} - 1) X_{ii} + \sum_{j=1}^{i-1} (D_{ij} - D_{ij+1}) X_{ij}\nonumber\\
&\ge D_{zz} X_{zz} + \sum_{j=1}^{z-1}(D_{zj} - D_{zj+1}) X_{zj}, \quad i=2, \ldots, n; \; z=1, \ldots, i - 1,\\
& 2 - y_i - y_z + 1 + (D_{zz} - 1) X_{zz} + \sum_{j=1}^{z-1}(D_{zj} - D_{zj+1}) X_{zj}\nonumber\\
&\ge D_{ii} X_{ii} + \sum_{j=1}^{i-1} (D_{ij} - D_{ij+1}) X_{ij}, \quad i=2, \ldots, n; \; z=1, \ldots, i - 1.
\end{align*}
\end{subequations}

Note that none of these constraints are necessary if the position of the change bin $t$ is fixed in advance. For example, given $t$, the valley trend constraints are replaced by two sets of constraints; one to guarantee a descending monotonic trend before $t$ and another to guarantee an ascending monotonic trend after $t$. Devising an effective heuristic to determine the optimal $t$ can yield probably optimal solutions while reducing the problem size significantly.

\subsubsection{Additional constraints}\label{section_additional_contraints}

\paragraph{Reduction of dominating bins.}

To prevent any particular bin from dominating the results, it might also be required that bins have at most a certain number of (total/non-event/event) records using constraints (\ref{min_max_bin_size} - \ref{min_max_bin_n_event}). Furthermore, we might produce more homogeneous solutions by reducing a concentration metric such as the standard deviation of the number of total/non-event/event records among bins. Three concentration metrics are considered: standard deviation, Herfindahl-Hirschman Index (HHI) \cite{Mironchyk2017} and the difference between the largest and smallest bin.

The standard deviation among the number of records for each bin is given by
\begin{equation*}
std = \left(\frac{1}{m - 1} \sum_{i=1}^n \left(\sum_{j=1}^i r_j X_{ij} - \frac{X_{ii}}{m} \sum_{i=1}^n \sum_{j=1}^i r_j X_{ij}\right)^2\right)^{1/2},
\end{equation*}
where $m = \sum_{i=1}^n X_{ii}$ is the optimal number of bins. Let us define the following auxiliary variables
\begin{equation*}
\mu = \frac{1}{m} \sum_{i=1}^n \sum_{j=1}^i r_j X_{ij}, \quad w_i = \sum_{j=1}^i r_j X_{ij} - \mu X_{ii}, \quad i=1, \ldots n.
\end{equation*}
Taking $w = (w_1, \ldots, w_n)^T$, the standard deviation $t$ can be incorporated to the formulation with a different representation. Since $std = (w^T w / (m - 1))^{1/2}$ then
\begin{equation*}
\frac{||w||_2}{(m - 1)^{1/2}} \le t \Longleftrightarrow ||w||_2^2 \le (m - 1) t^2.
\end{equation*}
The non-convex MINLP formulation using the parameter $\gamma$ to control the importance of the term $t$, 
\begin{subequations}
\begin{align}
\underset{X, \mu, w}{\text{max}} \quad & \sum_{i=1}^n V_{ii} X_{ii} + \sum_{j=1}^{i-1} (V_{ij} - V_{ij+1}) X_{ij} - \gamma t \label{reduce_std_objective}\\
\text{s.t.} \quad & \text{(\ref{sum_columns_one} - \ref{binary_indicator})}\\
& \sum_{i=1}^n w_i^2 \le (m - 1) t^2\\
& \mu = \frac{1}{m} \sum_{i=1}^n \sum_{j=1}^i r_j X_{ij}\\
& w_i = \sum_{j=1}^i r_j X_{ij} - \mu X_{ii}, & i=1, \ldots n\\
& m = \sum_{i=1}^n X_{ii}\\
& m \ge 0\\
& \mu \ge 0\\
& w_i \in \mathbb{R}, & i=1, \ldots, n.
\end{align}
\end{subequations}

A widely used metric to quantify concentration is HHI, which can be employed to asses the quality of a binning solution. Lower values of HHI correspond to more homogeneous bins. The HHI of the number of records for each bin is given by
\begin{equation*}
HHI = \frac{1}{r_T^2} \sum_{i=1}^n \left(\sum_{j=1}^i r_j X_{ij}\right)^2,
\end{equation*}
where $r_T = \sum_{i=1}^n r_i$ is the total number of records. The MIQP formulation using the parameter $\gamma$ to control the importance of HHI is stated as 
\begin{subequations}
\begin{align}
\underset{X}{\text{max}} \quad & \sum_{i=1}^n V_{ii} X_{ii} + \sum_{j=1}^{i-1} (V_{ij} - V_{ij+1}) X_{ij} - \frac{\gamma}{r_T^2} \sum_{i=1}^n \left(\sum_{j=1}^i r_j X_{ij}\right)^2 \label{reduce_hhi}\\
\text{s.t.} \quad & \text{(\ref{sum_columns_one} - \ref{binary_indicator})}
\end{align}
\end{subequations}

An effective MILP formulation can be devised using a simplification of the standard deviation approach based on reducing the difference between the largest and smallest bin. The MILP formulation is given by
\begin{subequations}
\begin{align}
\underset{X, p_{\min}, p_{\max}}{\text{max}} \quad & \sum_{i=1}^n V_{ii} X_{ii} + \sum_{j=1}^{i-1} (V_{ij} - V_{ij+1}) X_{ij} - \gamma (p_{\max} - p_{\min}) \label{reduce_diff_max_min_size}\\
\text{s.t.} \quad & \text{(\ref{sum_columns_one} - \ref{binary_indicator})}\\
& p_{\min} \le r_T (1 - X_{ii}) + \sum_{j=1}^i r_j X_{ij}, & i = 1, \ldots, n\\
& p_{\max} \ge \sum_{j=1}^i r_j X_{ij}, & i = 1, \ldots, n\\
& p_{\min} \le p_{\max}\\
& p_{\min} \ge 0.\\
& p_{\max} \ge 0.
\end{align}
\end{subequations}

\paragraph{Maximum p-value constraint.} A necessary constraint to guarantee that event rates between consecutive bins are statistically different is to impose a maximum p-value constraint setting a significance level $\alpha$. Suitable statistical tests are the Z-test, Pearson's Chi-square test or Fisher 's exact test. To perform these statistical tests we require an aggregated matrix of non-event and event records per bin, 
\begin{equation*}
R^{NE}_{ij} = \sum_{z=j}^i r_z^{NE}, \quad R^E_{ij} = \sum_{z=j}^i r_z^E, \quad i=1, \ldots, n;\; j=1, \ldots, i.
\end{equation*}

The preprocessing procedure to detect pairs of pre-bins that do not satisfy the p-value constraints using the Z-test is shown in Algorithm \ref{alg:max_pvalue}.
\begin{algorithm}[htbp]
\caption{Maximum p-value constraint using Z-test}\label{alg:max_pvalue}
\begin{algorithmic}[1]
\Procedure{p-value\_violation\_indices}{$n, R^{NE}, R^E, \alpha$}
\State $zscore = \Phi^{-1}(1 - \alpha / 2)$
\State $\mathcal{I} = \{\}$
\For {$i = 1,\ldots, n-1$}
	\State $l = i + 1$
	\For {$j = 1, \ldots, i$}
		\State $x = R^E_{ij}$
		\State $y = R^{NE}_{ij}$
		\For {$k = l, \ldots, n$}
			\State $w = R^E_{kl}$
			\State $z = R^{NE}_{kl}$
			\If {Z-test$(x, y, w, z) < zscore$}
				\State $\mathcal{I} = \mathcal{I} \cup (i, j, k, l)$
			\EndIf
		\EndFor
	\EndFor
\EndFor
\EndProcedure
\end{algorithmic}
\end{algorithm}

These constraints are added to the formulation by imposing that, at most, one of the bins violating the maximum p-value constraint can be selected. Two cases are considered depending on $j$:
\begin{equation*}
\begin{cases}
X_{ij} + X_{kl} \le 1 + X_{kl-1} & \text{ if } j = 1,\\
X_{ij} + X_{kl} \le 1 + X_{ij-1} + X_{kl-1} & \text{ if } j > 1
\end{cases}, \quad \forall(i, j, k, l) \in \mathcal{I}.
\end{equation*}
These two cases are represented in Figure \ref{violation_diagram}. Finally, note that previous constraints can be rewritten such that they resemble continuity constraints (\ref{flow_continuity}), respectively
\begin{equation*}
X_{ij} + (X_{kl} - X_{kl-1}) \le 1, \quad \text{and} \quad (X_{ij} - X_{ij-1}) + (X_{kl} - X_{kl-1}) \le 1.
\end{equation*}
\begin{figure}[ht]
	\centering
	\includegraphics[scale=0.75]{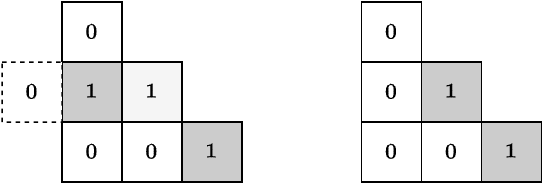}
	\caption{General violation constraint between two feasible solutions. Case $j=1$ (left). Case $j > 1$ (right).}
	\label{violation_diagram}
\end{figure}

\subsubsection{Mixed-integer programming reformulation for local and heuristic search}\label{section_localsolver}

The number of binary decision variables $X$ is $n (n + 1) / 2$. For large $n$ the $\mathcal{NP}$-hardness of the combinatorial optimization problem might limit the success of tree-search techniques. A first approach to tackle this limitation is reformulating the problem to reduce the number of decision variables. First, observe in Figure \ref{basic_diagram} that a solution is fully characterized by the diagonal of $X$. Thus, having the diagonal we can place the ones on the positions satisfying unique assignment (\ref{def_ct_sum_columns_one}) and continuity (\ref{def_ct_flow_continuity}) constraints. On the other hand, to return indexed elements in any aggregated matrix in the original formulation, we require the position of the first one by row. We note that this information can be retrieved by counting the number of consecutive zeros between ones (selected bins) of the diagonal. To perform this operation we use two auxiliary decision variables: an accumulator of preceding zeros $a_i$ and the preceding run-length of zeros $z_i$. A similar approach to counting consecutive ones is introduced in \cite{Kalvelagen2018}. The described approach is illustrated in Figure \ref{heuristic_diagram}.

\begin{figure}[ht]
	\centering
	\includegraphics[scale=0.33]{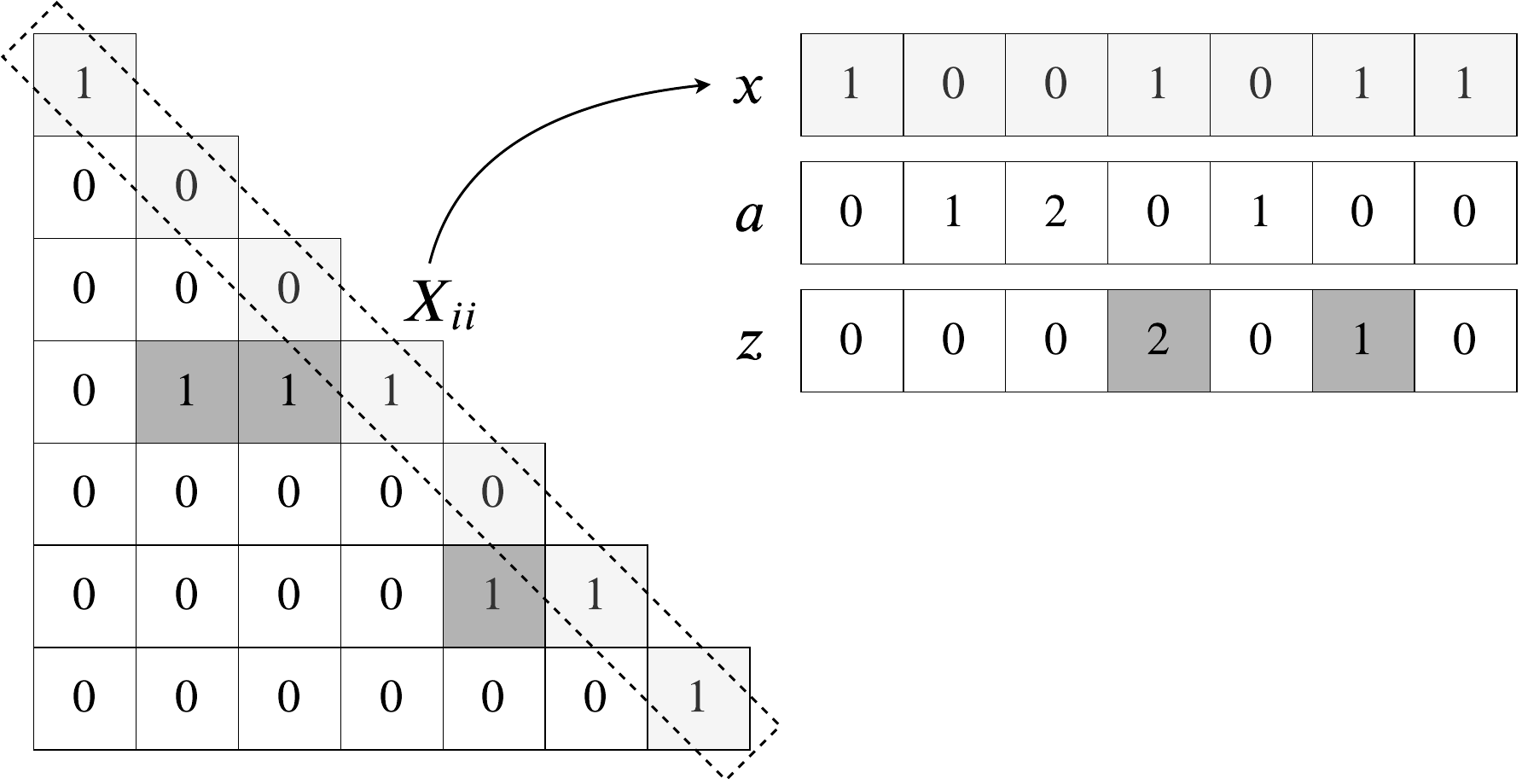}
	\caption{New decision variables suitable for counting consecutive zeros.}
	\label{heuristic_diagram}
\end{figure}

The positions in $z$ are zero-based indexes of the reversed rows of the aggregated matrices. For example, the aggregated lower triangular matrix $R^E$ is now computed backward: $R^E_{ij} = \sum_{z=i}^j r_z^E$ for $i=1, \ldots, n;\; j=1, \ldots, i$. Same for the aggregated matrices $V$, $D$, $R$ and $R^{NE}$. Let us define the parameters of the mathematical programming formulation:
\begin{align*}
& n \in \mathbb{N} && \textrm{number of pre-bins.}\\
& V_{[i, z_i]} \in \mathbb{R}^+_0 && \textrm{Information value.}\\
& D_{[i, z_i]} \in [0, 1] && \textrm{event rate.}\\
& R_{[i, z_i]} \in \mathbb{N} && \textrm{number of records.}\\
& R^{NE}_{[i, z_i]} \in \mathbb{N} && \textrm{number of non-event records.}\\
& R^E_{[i, z_i]} \in \mathbb{N} && \textrm{number of event records.}\\
& r^{NE}_{\min} \in \mathbb{N}  && \textrm{minimum number of non-event records per bins.}\\
& r^{NE}_{\max} \in \mathbb{N}  && \textrm{maximum number of non-event records per bins.}\\
& r^E_{\min} \in \mathbb{N}  && \textrm{minimum number of event records per bins.}\\
& r^E_{\max} \in \mathbb{N}  && \textrm{maximum number of event records per bins.}\\
& b_{\min} \in \mathbb{N}  && \textrm{minimum number of bins.}\\
& b_{\max} \in \mathbb{N}  && \textrm{maximum number of bins.}
\end{align*}
and the decision variables:
\begin{align*}
& x_i \in \{0, 1\} && \textrm{binary indicator variable.}\\
& a_i \in \mathbb{N}_0 && \textrm{accumulator of preceding zeros.}\\
& z_i \in \mathbb{N}_0 && \textrm{preceding run-length of zeros.}
\end{align*}

The new formulation with $3n$ decision variables is stated as follows
\begin{subequations}
\begin{align}
\underset{X}{\text{max}} \quad & \sum_{i=1}^n V_{[i, z_i]} x_i\\\label{new_objective_cp}
\text{s.t.} \quad & x_n = 1\\
\quad & a_i = (a_{i-1} + 1) (1-x_i), & i= 1, \ldots, n\\\label{accum_zero_cp}
\quad & z_i = a_{i-1} (1-x_{i-1}) x_i, & i= 1, \ldots, n\\\label{count_zero_cp}
\quad & b_{\min} \le \sum_{i=1}^n x_i \le b_{\max}\\
\quad & r_{\min} \le \sum_{i=1}^n R_{[i, z_i]} x_i \le r_{\max}, & i = 1, \ldots, n\\
\quad & r^{NE}_{\min} \le \sum_{i=1}^n R^{NE}_{[i, z_i]} x_i \le r^{NE}_{\max}, & i = 1, \ldots, n\\
\quad & r^E_{\min} \le \sum_{i=1}^n R^E_{[i, z_i]} x_i \le r^E_{\max}, & i = 1, \ldots, n\\
\quad & x_i \in \{0, 1\}, & i = 1, \ldots, n\\
\quad & a_i \in \mathbb{N}_0, & i = 1, \ldots, n\\
\quad & z_i \in \mathbb{N}_0, & i = 1, \ldots, n\label{z_cp}
\end{align}
\end{subequations}
This MINLP formulation is particularly suitable for Local Search (LS) and heuristic techniques, where decision variable $z_i$ can be used as an index, for example, $V_{[i, z_i]}$. The nonlinear constraints (\ref{accum_zero_cp}) and (\ref{count_zero_cp}) are needed for counting consecutive zeros. After the linearization of these constraints via big-$M$ inequalities or indicator constraints \cite{Kalvelagen2018}, the formulation is adequate for Constraint Programming (CP). Additional constraints such as monotonicity constraints can be incorporated to (\ref{new_objective_cp} - \ref{z_cp}) in a relatively simple manner: 

Monotonic trend ascending
\begin{equation*}
D_{[i, z_i]} x_i + 1 - x_i \ge D_{[j, z_j]} x_j + \beta (x_i + x_j - 1), \quad i=2, \ldots, n;\; z=1, \ldots, i - 1.
\end{equation*}

Monotonic trend descending
\begin{equation*}
D_{[i, z_i]} x_i + \beta (x_i + x_j - 1) \le 1 - x_j + D_{[j, z_j]} x_j, \quad i=2, \ldots, n;\; z=1, \ldots, i - 1.
\end{equation*}

Monotonic trend concave
\begin{equation*}
-D_{[i, z_i]} x_i + 2 D_{[j, z_j]} x_j - D_{[k, z_k]} x_k \ge x_i + x_j + x_k - 3,
\end{equation*}
for $i=3, \ldots, n;\; j=2, \ldots, i - 1;\; k=1, \ldots, j - 1$.

Monotonic trend convex
\begin{equation*}
D_{[i, z_i]} x_i - 2 D_{[j, z_j]} x_j + D_{[k, z_k]} x_k \ge x_i + x_j + x_k - 3,
\end{equation*}
for $i=3, \ldots, n;\; j=2, \ldots, i - 1;\; k=1, \ldots, j - 1$.

Monotonic trend peak: constraints (\ref{disjoint_constraint} - \ref{disjoint_aux_variables}) and 
\begin{subequations}
\begin{align*}
y_i + y_j + 1 + (D_{[j, z_j]} - 1) x_j - D_{[i, z_i]} x_i &\ge 0\\
2 - y_i -y_j + 1 + (D_{[i, z_i]} - 1) x_i - D_{[j, z_j]} x_j &\ge 0,
\end{align*}
\end{subequations}
for $i=2, \ldots, n;\; z=1, \ldots, i - 1$.

Monotonic trend valley: constraints (\ref{disjoint_constraint} - \ref{disjoint_aux_variables}) and 
\begin{subequations}
\begin{align*}
y_i + y_j + 1 + (D_{[i, z_i]} - 1) x_i - D_{[j, z_j]} x_j & \ge 0\\
2 - y_i -y_j + 1 + (D_{[j, z_j]} - 1) x_j - D_{[i, z_i]} x_i &\ge 0
\end{align*}
\end{subequations}
for $i=2, \ldots, n;\; z=1, \ldots, i - 1$.

In Section \ref{section_experiments}, we compare the initial CP/MIP formulation to the presented LS formulation for large size instances. 

\subsection{Mixed-integer programming formulation for continuous target}

The presented optimal binning formulation given a binary target can be seamlessly extended to a continuous target. Following the methodology developed for Equations (\ref{iv_obj_params}) and (\ref{iv_obj_definition}), we could adapt the IV statistic for a continuous target as described in \cite{FICO2014}
\begin{equation*}
IV = \sum_{i=1}^n \left|\mu - u_{i} \right| \frac{r_i}{r_T},
\end{equation*}
where $\mu \in \mathbb{R}$ is the target mean for all records (global target mean), $r_T$ is total number of records, and $u_i$ and $r_i$ are the target mean and number of records for each bin, respectively. The goal of this metric is to obtain bins with a target mean as different as possible from the global target mean.  However, we note that the usage of this metric as objective function might tend to reduce granularity by selecting only a few bins with large differences. Therefore, we use only the $p$-norm distance ($L_1$-norm of $L_2$-norm) term and exclude the relative number of records term. An aggregated lower triangular matrix $L_{ij} \in \mathbb{R}^+_0, \forall (i, j) \in \{1, \ldots, n: i \ge j\}$ can be pre-computed as follows,
\begin{equation*}
L_{ij} = \left\|\mu - U_{ij}\right\|_p, \quad U_{ij} = \frac{\sum_{z=i}^j s_z}{\sum_{z=i}^j r_z},
\end{equation*}
where $U_{ij} \in \mathbb{R}$ is the aggregated matrix of target mean values, and $s_i$ is the sum of target values for each pre-bin. A more robust approach might replace the mean by an order statistic, but unfortunately, the aggregated matrix computed from the pre-binning data would require approximation methods. Finally, replacing $L_{ij}$ in the objective function (\ref{objective}), the resulting formulation is given by
\begin{subequations}
\begin{align}
\underset{X}{\text{max}} \quad & \sum_{i=1}^n L_{ii} X_{ii} + \sum_{j=1}^{i-1} (L_{ij} - L_{ij+1}) X_{ij}\\
\text{s.t.} \quad & \text{(\ref{sum_columns_one} - \ref{binary_indicator})}
\end{align}
\end{subequations}

\subsubsection{Monotonicity constraints}
As for the binary target case, we can impose monotonicity constraints between the mean value of consecutive bins. Since establishing tight bounds for $U_{ij} \in \mathbb{R}$ is not trivial, we discard a big-$M$ formulation. Another traditional technique such as the use of SOS1 sets is also discarded to avoid extra variables and constraints. Instead, we state the ascending monotonic trend in double implication form as follows 
\begin{align*}
&X_{ii} = 1 \textrm{ and } X_{zz} = 1 \Longrightarrow
U_{zz} X_{zz} + \sum_{j=1}^{z-1}(U_{zj} - U_{z j+1}) X_{zj} + \beta\nonumber\\ &\le U_{ii} X_{ii} + \sum_{j=1}^{i - 1} (U_{ij} - U_{ij+1})X_{ij}, \quad i=2, \ldots, n; \; z=1,\ldots i-1.
\end{align*}
The enforced constraint must be satisfied iff the two literals $X_{ii}$ and $X_{zz}$ are true, otherwise the constraint is ignored. This is half-reified linear constraint \cite{Feydy2011}. The parameter $\beta$ is the minimum mean difference between consecutive bins. Similarly, for the descending constraint,
\begin{align*}
&X_{ii} = 1 \textrm{ and } X_{zz} = 1 \Longrightarrow U_{ii} X_{ii} + \sum_{j=1}^{i-1}(U_{ij} - U_{i j+1})X_{ij} + \beta\nonumber\\ 
&\le U_{zz} X_{zz} + \sum_{j=1}^{z - 1} (U_{zj} - U_{zj+1})X_{zj}, \quad i=2, \ldots, n; \; z=1,\ldots i-1.
\end{align*}

Furthermore, the concave and convex trend can be written in triple implication form using literals $X_{ii}$, $X_{jj}$ and $X_{kk}$. The concave trend constraints are
\begin{align*}
X_{ii} = 1,  X_{zz} = 1 \textrm{ and } X_{kk} = 1 \Longrightarrow 
&-\left(U_{ii}X_{ii} + \sum_{z=1}^{i-1}(U_{iz} - U_{iz+1}) X_{iz}\right)\nonumber\\ 
& + 2 \left(U_{jj}X_{jj} + \sum_{z=1}^{j-1}(U_{jz} - U_{jz+1}) X_{jz}\right)\nonumber\\
&-\left(U_{kk}X_{kk} + \sum_{z=1}^{k-1}(U_{kz} - U_{kz+1}) X_{kz}\right) \ge 0,
\end{align*}
for $i=3, \ldots n$; $j=2, \ldots, i-1$ and $k=1, \ldots, j-1$. Similarly, for convex trend we get
\begin{align*}
X_{ii} = 1,  X_{zz} = 1 \textrm{ and } X_{kk} = 1 \Longrightarrow 
&\left(U_{ii}X_{ii} + \sum_{z=1}^{i-1}(U_{iz} - U_{iz+1}) X_{iz}\right)\nonumber\\ 
& - 2 \left(U_{jj}X_{jj} + \sum_{z=1}^{j-1}(U_{jz} - U_{jz+1}) X_{jz}\right)\nonumber\\
&\left(U_{kk}X_{kk} + \sum_{z=1}^{k-1}(U_{kz} - U_{kz+1}) X_{kz}\right) \ge 0,
\end{align*}
for $i=3, \ldots n$; $j=2, \ldots, i-1$ and $k=1, \ldots, j-1$.

The formulation can be extended to support valley and peak trend. The peak trend requires constraints (\ref{disjoint_constraint} - \ref{disjoint_aux_variables}) and 
\begin{subequations}
\begin{align*}
X_{ii} = 1 \textrm{ and } X_{zz} = 1 \Longrightarrow
&M (y_i + y_z) + U_{zz} X_{zz} + \sum_{j=1}^{z-1} (U_{zj} - U_{zj+1}) X_{zj}\nonumber\\
&\ge U_{ii} X_{ii} + \sum_{j=1}^{i-1}(U_{ij} - U_{ij+1}) X_{ij}, \\
& M (2 - y_i - y_z) + U_{ii} X_{ii} + \sum_{j=1}^{i-1}(U_{ij} - U_{ij+1}) X_{ij}\nonumber\\
&\ge U_{zz} X_{zz} + \sum_{j=1}^{z-1} (U_{zj} - U_{zj+1}) X_{zj},
\end{align*}
\end{subequations}
for $i=2, \ldots, n; \; z=1, \ldots, i - 1$. The big-$M$ formulation to handle disjoint constraints in (\ref{disjoint_constraint} - \ref{disjoint_aux_variables}) requires an effective bound, we suggest $M= \max(\{|U_{ij}|: i=1, \ldots, n: i \ge j\})$. Similarly, for the valley trend we include constraints  and 
\begin{subequations}
\begin{align*}
X_{ii} = 1 \textrm{ and } X_{zz} = 1 \Longrightarrow
&M (y_i + y_z) + U_{ii} X_{ii} + \sum_{j=1}^{i-1} (U_{ij} - U_{ij+1}) X_{ij}\nonumber\\
&\ge U_{zz} X_{zz} + \sum_{j=1}^{z-1}(U_{zj} - U_{zj+1}) X_{zj},\\
& M(2 - y_i - y_z) + U_{zz} X_{zz} + \sum_{j=1}^{z-1}(U_{zj} - U_{zj+1}) X_{zj},\nonumber\\
&\ge U_{ii} X_{ii} + \sum_{j=1}^{i-1} (U_{ij} - U_{ij+1}) X_{ij},
\end{align*}
\end{subequations}
for $i=2, \ldots, n; \; z=1, \ldots, i - 1$.

\subsubsection{Additional constraints}

\paragraph{Maximum p-value constraint.} Similar to the binary target case, we can impose a maximum p-value constraint between consecutive bins to ensure that their means are statistically different. In this case, the T-test for means is appropriate, although it requires the standard deviation. To compute an aggregate matrix of standard deviations, $SD_{ij}$, we could use 
\begin{equation*}
SD_{ij} = \frac{\sum_{z=i}^j ss_z}{\sum_{z=i}^j r_z} - \left(\frac{\sum_{z=i}^j s_z}{\sum_{z=i}^j r_z}\right)^2,
\end{equation*}
where $ss_i$ is the sum of squared target values for each pre-bin.

Finally, the algorithm and procedure described in Section \ref{section_additional_contraints} to incorporate these constraints can be readily reused.

\subsection{Mixed-integer programming formulation for multi-class target}

A simple approach to support a multi-class target is to use the one-vs-rest scheme with $n_C$ distinct classes. This scheme consists of building a binary target for each class. The resulting mathematical formulation closely follows the formulation for binary target,
\begin{subequations}
\begin{align}
\underset{X}{\text{max}} \quad & \sum_{c=1}^{n_C}\sum_{i=1}^n V^c_{ii} X_{ii} + \sum_{j=1}^{i-1} (V^c_{ij} - V^c_{ij+1}) X_{ij} \label{multiclass_objective}\\
\text{s.t.} \quad & \text{(\ref{sum_columns_one} - \ref{min_max_bin_size})}
\end{align}
\end{subequations}
Note that for this formulation we need an aggregated matrix $V$ and $D$ for each class $c$. It is important to emphasize that the monotonicity constraints in Section \ref{section_monotonicity} act as linking constraints among classes, otherwise, $n_C$ optimal binning problems with binary target could be solved separately. Again, additional constraints described in Section \ref{section_additional_contraints} can be naturally incorporated with minor changes.

\section{Algorithmic details and implementation}\label{section_algorithms}

\subsection{Automatic monotonic trend algorithm}
Our approach to automate the monotonic trend decision employs a Machine Learning (ML) classifier that predicts, given the pre-binning data, the most suitable monotonic trend to maximize discriminatory power. In particular, we aim to integrate an off-line classifier, hence we are merely interested in ML classification algorithms easily embeddable. Recently, a similar approach was implemented in the commercial solver CPLEX  to make automatic decisions over some algorithmic choices \cite{bonami2018}.

For this study, a dataset is generated with 417 instances collected from public datasets. We design a set of 16 numerical features, describing the pre-binning instances in terms of number of pre-bins, distribution of records per pre-bin and trend features. The most relevant trend features are: number of trend change points, linear regression coefficient sense, area of the convex hull, and area comprised among extreme trend points.

The labeling procedure consists of solving all instances selecting the ascending (A), descending (D), peak (P) and valley (V) monotonic trend. Concave and convex trends are discarded due to being a special case of peak and valley, respectively. In what follows, without loss of generality, we state the procedure for the binary target case: if the relative difference between IV with ascending/descending monotonic trend and IV with peak/valley trend is less than 10\%, the ascending/descending monotonic trend is selected, due to the lesser resolution times. Table \ref{table_dataset} summarizes the composition of the dataset with respect to assigned labels. We note that the dataset is slightly unbalanced, being predominant the descending label (D).

\begin{table}[htpb]
	\centering
	\scalebox{0.8}{
\begin{tabular}{ccc}
\hline
\bf{Label} & \bf{Instances} & \bf{Frequency (\%)}\\
\hline
	A & 84 & 20\\
	D & 200 & 48\\
	P & 76 & 18\\
	V & 57 & 14\\
	\hline
	\bf{Total} & \bf{417}\\
	\hline
\end{tabular}}
\caption{Number of instances and percentage for each label.}
\label{table_dataset}
\end{table}

To perform experiments, the dataset is split into train and test subsets in a stratified manner to treat unbalanced data. The proportion of the test set is 30\%. Three interpretable multi-class classification algorithms are tested, namely, logistic regression, decision trees (CART) and Support vector Machine (SVM) using the Python library Scikit-learn \cite{scikit-learn}. All three algorithms are trained using option \texttt{class\_weight="balanced"}. Throughout the learning process, we discard 8 features and perform hyperparameter optimization for all three algorithms. These experiments show that SVM and CART have similar classification measures, and we decide to choose CART (max depth 5) to ease implementation.

On the test set, the trained CART has a weighted average accuracy, precision and recall of 88\%. See classification measures and the confusion matrix in Table \ref{table_measures_confusion}. We observe that various instances of the minority class (V) are misclassified, indicating that more instances or new features might be required to improve classification measures. Improving this classifier is part of ongoing research.

\begin{table}[htpb]
	\centering
	\scalebox{0.8}{
\begin{tabular}{ll}
\begin{tabular}{ccccc}
\hline
\bf{Label} & \bf{Precision} & \bf{Recall} & \bf{F1-score} & \bf{Support}\\
\hline
	A & 0.85 & 0.88 & 0.86 & 25\\
	D & 0.97 & 0.93 & 0.95 & 61\\
	P & 0.81 & 0.88 & 0.88 & 23\\
	V & 0.71 & 0.59 & 0.65 & 17\\
	\hline
weighted avg & 0.88 & 0.88 & 0.88 & 126\\
	\hline
\end{tabular}
\hspace{1cm}
\begin{tabular}{ccccc}
\hline
& \bf{A} & \bf{D} & \bf{P} & \bf{V}\\
\hline
\bf{A} & 22 & 0 & 1 & 2\\
\bf{D} & 0 & 57 & 2 & 2\\
\bf{P} & 1 & 0 & 22 & 0\\
\bf{V} & 3 & 2 & 2 & 10\\
\hline
\end{tabular}
\end{tabular}}
\caption{Classification measures (left) and confusion matrix (right) for CART on the test set.}
\label{table_measures_confusion}
\end{table}

\subsection{Presolving algorithm}

The mathematical programming formulation is a hard combinatorial optimization problem that does not scale well as the number of pre-bins increases. To reduce solution times we need to reduce the search space to avoid deep tree searches during branching. The idea is to develop a presolving algorithm to fix bins not satisfying monotonicity constraints, after that the default presolver may be able to reduce the problem size significantly. The presolving algorithm applies to the binary target case and was developed after several observations about the aggregated matrix of event rates $D$. Algorithm \ref{alg:asc_mono} shows the implemented approach for the ascending monotonicity trend. Presolving algorithm for the descending monotonicity is analogous, only requiring inequalities change.

\begin{algorithm}[H]
\caption{Preprocessing ascending monotonicity}\label{alg:asc_mono}
\begin{algorithmic}[1]
\Procedure{PreprocessingAscending}{$D, X, \beta$}
\For {$i = 1,\ldots, n-1$}
	\If {$D_{i+1, i} - D_{i+1, i+1} > 0$}
		\State fix $X_{i,i} = 0$
	\EndIf
	\For {$j = 1,\ldots, n-i-1$}
		\If {$D_{i+1+j, i} - D_{i+1+j, i+1+j} > 0$}
			\State fix $X_{i+j, i+j} = 0$
		\EndIf
	\EndFor
\EndFor
\EndProcedure
\end{algorithmic}
\end{algorithm}

\subsection{Binning quality score}
To assess the quality of binning for binary target, we develop a binning quality score considering the following aspects:
\begin{itemize}
\item Predictive power: IV rule of thumb \cite{Siddiqi2005} in Table \ref{IV_rule_of_thumb}.
\item Statistical significance: bin event rates must be statistically different, therefore large p-values penalize the quality score.
\item Homogeneity: binning with homogeneous bin sizes or uniform representativeness, increases reliability.
\end{itemize}

\begin{table}[H]
	\centering
	\scalebox{0.8}{
\begin{tabular}{cccccccc}
\hline
\bf{IV}  & \bf{predictive power}\\
\hline
$[0,  0.02)$ & not useful\\
$[0.02, 0.1)$ & weak\\
$[0.1, 0.3)$ & medium\\
$[0.3, 0.5)$ & strong\\
$[0.5,  \infty)$ & over-prediction\\
\hline
	\end{tabular}}
	\caption{Information Value rule of thumb.}
	\label{IV_rule_of_thumb}
\end{table}

To account for all these aspects, we propose a rigorous binning quality score function
\begin{proposition}
Given a binning with Information Value $\nu$, p-values between consecutive bins $p_i$, $i=1, \ldots, n - 1$ and normalized bins size $s_i$, $i=1, \ldots, n$, the binning quality score function is defined as
\begin{equation}
Q(\nu, p, s) = \frac{\nu}{c} \exp\left(-\nu^2/ (2c^2) + 1/2\right) \left(\prod_{i=1}^{n-1} (1 - p_i)\right) \left(\frac{1 - \sum_{i=1}^n s_i^2}{1 - 1/n}\right),
\end{equation}
where $Q(\nu, p, s) \in [0, 1]$ and $c = \frac{1}{5}\sqrt{\frac{2}{\log(5/3)}}$ is the best a priori IV value in $[0.3, 0.5)$.
\end{proposition}
\begin{proof}
Given the rule of thumb in Table \ref{IV_rule_of_thumb}, let us consider the set of statistical distributions with positive skewness, positive fat-tail, and support on the semi-infinite interval $[0, \infty)$. The function should penalize large values of Information Value $\nu$, and fast decay is expected after a certain threshold indicating over-prediction. This fast decay is a required property that must be accompanied by the following statement: $\lim_{\nu\to 0} f(\nu) = \lim_{\nu \to \infty} f(\nu) = 0$ Among the available distributions satisfying aforementioned properties, we select the Rayleigh distribution, which probability density function is given by
\begin{equation*}
f(\nu;c) = \frac{\nu}{c^2} e^{-\nu^2 / (2c^2)}, \quad \nu\ge 0.
\end{equation*}

This is a statistical distribution, not a function, hence we need a scaling factor so that $\max_{\nu \in [0, \infty)}f(\nu; c) = 1$: the maximum value of a unimodal probability distribution is the mode $c$, thus
\begin{equation*}
\gamma = f(c,c) = \frac{1}{c \sqrt{e}} \Longrightarrow \frac{f(\nu,c)}{\gamma} =  \frac{\nu \exp\left(-\nu^2/ (2c^2) + 1/2\right)}{c}.
\end{equation*}
The optimal $c$ such that $f(a) = f(b)$ for $b > a$ can be obtained by solving $f(b; c) - f(a; c) = 0$ for $c$, which yields
\begin{equation*}
c^* = \frac{\sqrt{b^2 - a^2}}{\sqrt{2\log(b/a)}}.
\end{equation*}
Term $\prod_{i=1}^{n-1} (1 - p_i)$ assesses the statistical significance of the bins. Furthermore, term $\frac{1 - \sum_{i=1}^n s_i^2}{1 - 1/n} = 1 - HHI^*$, where $HHI^*$ is the normalized Herfindahl Hirschman Index, assesses the homogeneity/uniformity of the bin sizes.
\end{proof}

For example, if we consider that the boundaries of the interval with strong IV predictive power in Table \ref{IV_rule_of_thumb}, $a = 0.3$ and $b = 0.5$, should produce the same quality score, $c^*(a,b) = \frac{1}{5}\sqrt{\frac{2}{\log(5/3)}}$. Table \ref{table_iv_metrics} shows the value of $f(\nu, c^*)$ for various IV values $\nu$; note the fast decay of $f(\nu, c^*)$ when $\nu > 0.5$.
\begin{table}[H]
	\centering
	\scalebox{0.8}{
\begin{tabular}{cccccccccc}
\hline
\textbf{$\nu$}  & \boldmath{$0$} &  \boldmath{$0.02$} &  \boldmath{$0.1$} &  \boldmath{$0.3$} &  \boldmath{$0.5$} &  \boldmath{$0.7$}  &  \boldmath{$0.9$}  &  \boldmath{$1$}  &  \boldmath{$1.5$}\\
\hline
	$f(\nu, c^*)$ & 0 & 0.083 & 0.404 & 0.938 & 0.938 & 0.610 & 0.282 & 0.171 & 0.005\\
	\hline
\end{tabular}}
\caption{Function values for various $\nu$ values.}
\label{table_iv_metrics}
\end{table}

\subsection{Implementation}

The presented mathematical programming formulations are implemented using Google OR-Tools \cite{ortools} with the open-source MILP solver CBC \cite{cbc_2018}, and Google's BOP and CP-SAT solvers. Besides, the specialized formulation in Section \ref{section_localsolver} is implemented using the commercial solver LocalSolver \cite{localsolver2011}. The python library OptBinning\footnote{\url{https://github.com/guillermo-navas-palencia/optbinning}} has been developed throughout this work to ease usability and reproducibility.

Much of the implementation effort focuses on the careful implementation of constraints and the development of fast algorithms for preprocessing and generating the model data. A key preprocessing algorithm is a pre-binning refinement developed to guarantee that no bins have 0 non-events and events records in the binary target case.

Categorical variables require special treatment: pre-bins are ordered in ascending order with respect to a given metric; the event rate for binary target and the target mean for a continuous target. The original data is replaced by the ordered indexes and is then used as a numerical (ordinal) variable. Furthermore, during preprocessing, the non-representative categories may be binned into an ``others'' bin. Similarly, missing values and special values are incorporated naturally as additional bins after the optimal binning is terminated.

\section{Experiments}\label{section_experiments}

The experiments were run on an Intel(R) Core(TM) i5-3317 CPU at 1.70GHz, using a single core, running Linux. Two binning examples are shown in Tables \ref{fico_example} and \ref{kaggle_example}, using Fair Isaac (FICO) credit risk dataset \cite{FICO2018} ($N=10459$) and Home Credit Default Risk Kaggle competition dataset \cite{HomeCreditGroup2018} ($N=307511$), respectively.

Example in Table \ref{fico_example} uses the variable AverageMInFile (Average Months in File) as an risk driver. FICO dataset imposes monotonicity constraints to some variables, in particular, for this variable, the event rate must be monotonically decreasing. Moreover, the dataset includes three special values/codes defined as follows:
\begin{itemize}
\item -9: No Bureau Record or No Investigation
\item -8: No Usable/Valid Trades or Inquiries
\item -7: Condition not Met (e.g. No Inquiries, No Delinquencies)
\end{itemize}
For the sake of completeness, we also include a few random missing values on the dataset. As shown in Table \ref{fico_example}, these values are separately treated by incorporating a Special and Missing bin. Regarding computation time, this optimal binning instance is solved in 0.08 seconds. The optimization time accounts for 91\% of the total time, followed by the pre-binning time representing about 6\%. The remaining 3\% is spent in pre-processing and post-processing operations.

\begin{table}[htpb]
	\centering
	\scalebox{0.75}{
\begin{tabular}{cccccccccc}
\hline
\bf{Bin}  & \bf{Count} & \bf{Count (\%)} & \bf{Non-event} & \bf{Event} & \bf{Event rate} & \bf{WoE} & \bf{IV} & \bf{JS}\\
\hline
$(-\infty,  30.5)$ & 544 & 0.052013 & 99 & 445 & 0.818015 & -1.41513 & 0.087337	& 0.010089\\
$[30.5, 48.5)$ & 1060 & 0.101348 & 286 & 774 &	0.730189 & -0.907752	& 0.076782	& 0.009281\\
$[48.5, 54.5)$ & 528&	0.050483&	184&	344&	0.651515	&-0.537878&	0.014101	&0.001742\\
$[54.5, 64.5)$ &	1099	&0.105077	&450&	649	&0.590537	&-0.278357	&0.008041	&0.001002\\
$[64.5, 70.5)$&	791	&0.075629	&369&	422	&0.533502	&-0.046381	&0.000162	&0.000020\\
$[70.5, 74.5)$ &	536	&0.051248	&262&	274	&0.511194	&0.0430441	&0.000095	&0.000012\\
$[74.5, 81.5)$ 	&912 &	0.087198&	475	&437&	0.479167	&0.171209	&0.002559	&0.000320\\
$[81.5, 101.5)$	&2009	&0.192083	&1141	&868&	0.432056	&0.361296	&0.025000	&0.003108\\
$[101.5, 116.5)$	& 848	&0.081078	&532&	316	&0.372642	&0.608729	&0.029532	&0.003636\\
$[116.5, \infty)$	&1084	&0.103643	&702&	382	&0.352399	&0.696341	&0.049039	&0.006009\\
Special & 558	&0.053351	&252&	306	&0.548387	&-0.106328	&0.000601	&0.000075\\
Missing & 490 &	0.046850&	248	&242&	0.493878&	0.112319&	0.000592	&0.000074\\
\hline
	\end{tabular}}
	\caption{Example optimal binning using variable AverageMInFile from FICO dataset.}
	\label{fico_example}
\end{table}

Example in Table \ref{kaggle_example} uses the categorical variable ORGANIZATION\_TYPE from the Kaggle dataset. This variable has 58 categories, and we set the non-representative categories cut-off to $0.01$. Note that the bin just before the Special bin corresponds to the bin with non-representative categories, which is excluded from the optimization problem, hence monotonicity constraint does not apply. This optimal binning instance is solved in 0.25 seconds. For categorical variables, most of the time is spent on pre-processing, 71\% in this particular case, whereas the optimization problem is solved generally faster.

\begin{table}[H]
	\centering
	\scalebox{0.75}{
\begin{tabular}{cccccccccc}
\hline
\bf{Bin}  & \bf{Count} & \bf{Count (\%)} & \bf{Non-event} & \bf{Event} & \bf{Event rate} & \bf{WoE} & \bf{IV} & \bf{JS}\\
\hline
[XNA, School] & 64267 &0.208991 &60751	&3516	&0.054709	&0.416974	&0.030554	&0.003792\\
$[$Medicine, ...$]$ & 31845	& 0.103557	& 29673	& 2172	& 0.068205	& 0.182104 & 	0.003182 & 	0.000397\\
$[$Other$]$ & 16683 &	0.054252& 15408 &	1275	 & 0.076425	& 0.0594551 &	0.000187	& 0.000023\\
$[$Business ...$]$ & 16537	&0.053777 &	15150	& 1387	& 0.083873	& -0.0416281 &	0.000095&	0.000012\\
$[$Transport: ...$]$ & 81221	& 0.264124	&73657	&7564	&0.093129	&-0.156466	&0.006905	&0.000862\\
$[$Security, ...$]$ &55150	&0.179343	&49424	&5726	&0.103826	&-0.277067	&0.015465	&0.001927\\
$[$Housing, ...$]$ & 41808	&0.135956	&38623	&3185	&0.076182	&0.0629101	&0.000524	&0.000065\\
Special & 0 & 0 & 0 & 0 & 0 & 0 & 0 & 0\\
Missing & 0 & 0 & 0 & 0 & 0 & 0 & 0 & 0\\
\hline
	\end{tabular}}
	\caption{Example optimal binning using categorical variable ORGANIZATION\_TYPE from Home Credit Default Risk Kaggle competition dataset.}
	\label{kaggle_example}
\end{table}

\subsection{Benchmark CP/MIP vs local search heuristic}

For large instances, we compare the performance of Google OR-Tools' solvers BOP (MIP) and CP-SAT against LocalSolver. For these tests we select two variables from Home Credit Default Risk Kaggle competition dataset \cite{HomeCreditGroup2018} ($N=307511$). We aim to perform a far finer binning than typical in many applications to stress the performance of classical solvers for large combinatorial optimization problems.

Tables \ref{benchmark_preprocessing_1} and \ref{benchmark_preprocessing_2} show results for varying number of pre-bins $n$ and monotonic trends. In test 1 from Table \ref{benchmark_preprocessing_1} , LocalSolver does not improve after 10 seconds, not being able to reduce the optimality gap. In test 2, LocalSolver outperforms CP-SAT, finding the optimal solution after 5 seconds, 28x faster. In test 3, solution times are comparable. Results reported in Table \ref{benchmark_preprocessing_2} are also interesting; in test 1, BOP and CP-SAT solvers cannot find an optimal solution after 1000 seconds. LocalSolver finds the best found feasible solution after 30 seconds. Nevertheless, we recall that the described heuristic for peak/valley trend introduced in Section \ref{section_monotonicity} could reduce resolution times substantially, obtaining times comparable to those when choosing ascending/descending monotonic trend.

\begin{table}[htpb]
	\centering
	\scalebox{0.8}{
\begin{tabular}{cccccccc}
\hline
\bf{n} &\bf{monotonic trend} & \bf{solver} & \bf{variables} &\bf{constraints} &  \bf{time} &\bf{solution} & \bf{gap}\\
\hline
48 & peak & cp & 1225 & 3528 & 12.7 & 0.03757878 & -\\
48 & peak & ls & 193 & 2352 & 1 & 0.03373904 & 10.2\%\\
48 & peak & ls & 193 & 2352 & 5 & 0.03386574 & 9.9\%\\
48 & peak & ls & 193 & 2352 & 10 & 0.03725560 & 0.9\%\\
\hline
77 & peak & cp & 3081 & 9009 & 140.9 & 0.03776231 & -\\
77 & peak & ls &  309 & 6007 &     1 & 0.03078212 & 18.5\%\\
77 & peak & ls &  309 & 6007 &     5 & 0.03776231 & 0.0\%\\
\hline
77 & descending & cp & 3003 & 6706 & 0.9 & 0.03386574 & - \\
77 & descending & ls & 231 & 2969 & 1 & 0.03386574 & 0.0\%\\
\hline
	\end{tabular}}
	\caption{Variable REGION\_POPULATION\_RELATIVE. Performance comparison Google OR-Tools' CP-SAT vs LocalSolver. Time in seconds.}
	\label{benchmark_preprocessing_1}
\end{table}

\begin{table}[htpb]
	\centering
	\scalebox{0.8}{
\begin{tabular}{cccccccc}
\hline
\bf{n} &\bf{monotonic trend} & \bf{solver} & \bf{variables} &\bf{constraints} &  \bf{time} &\bf{solution} & \bf{gap}\\
\hline
100 & peak & cp & 5151 & 15150 & t & 0.11721972 & -\\
100 & peak & mip & 5151 & 15150 & t & $0.11786335^{*}$ & -\\
100 & peak & ls & 401 & 10101 & 1 & 0.11666556 & 1.0\%\\
100 & peak & ls & 401 & 10101 & 5 & 0.11735812 & 0.4\%\\
100 & peak & ls & 401 & 10101 & 10 & 0.11771822 & 0.1\%\\
100 & peak & ls & 401 & 10101 & 30 & 0.11786335 & 0.0\%\\
\hline
100 & ascending & cp & 5050 & 10933 & 2.3 & 0.05175782 & - \\
100 & ascending & ls & 300 & 5000 & 1 & 0.05175782 & 0.0\%\\
\hline
	\end{tabular}}
	\caption{Variable DAYS\_EMPLOYED. Performance comparison Google OR-Tools' CP-SAT/BOP vs LocalSolver. Time in seconds. *: Best feasible solution. t: 1000 seconds exceeded.}
	\label{benchmark_preprocessing_2}
\end{table}

\section{Conclusions}\label{section_conclusions}
We propose a rigorous and flexible mathematical programming formulation to compute the optimal binning. This is the first optimal binning algorithm to achieve solutions for nontrivial constraints, supporting binary, continuous and multi-class target, and handling several monotonic trends rigorously. Importantly, the size of the decision variables and constraints used in the presented formulations is independent of the size of the datasets; they are entirely controlled by the starting solution computed during the pre-binning process. In the future, we plan to extend our methodology to piecewise-linear binning and multivariate binning. Lastly, the code is available at \url{https://github.com/guillermo-navas-palencia/optbinning} to ease reproducibility.

\bibliographystyle{plain}
\bibliography{optbinning_bib}

\end{document}